\documentclass[letterpaper]{article}
\usepackage{aaai24}
\nocopyright
\usepackage{times}
\usepackage{helvet}
\usepackage{courier}
\usepackage[hyphens]{url}
\usepackage{graphicx}
\usepackage{natbib}
\usepackage{caption}
\usepackage{algorithm}
\usepackage{algorithmic}
\usepackage{newfloat}
\usepackage{listings}
\DeclareCaptionStyle{ruled}{labelfont=normalfont,labelsep=colon,strut=off} 
\floatstyle{ruled}
\newfloat{listing}{tb}{lst}{}
\floatname{listing}{Listing}

\usepackage{graphicx}
\usepackage{amsmath}
\usepackage{amsthm}
\usepackage{amsfonts}
\usepackage{xcolor}
\usepackage{mathtools}

\newtheorem{theorem}{Theorem}

\newtheorem{lemma}{Lemma}
\newtheorem{definition}{Definition}

\newtheorem{example}{Example}

\usepackage{tikz}
\usepackage{amssymb}
\usepackage{dsfont}
\usepackage{url}

\usetikzlibrary{automata,positioning,decorations.markings,arrows,intersections,shapes,%
  shapes.symbols,calc,backgrounds,shadings}

\tikzset{
  >=stealth,
  box state/.style={draw,rectangle,minimum size=8mm},
  prob state/.style={draw,very thick,shape=circle,darkblue,minimum size=3mm,inner sep=0mm},
  node distance=2cm,on grid,auto, initial text=,
  every loop/.style={shorten >=0pt},
  accepting/.style={double distance=1.2pt, outer sep = 0.6pt+\pgflinewidth},
  accepting dot/.style={above=-2.5pt,circle,fill,darkgreen,inner sep=2pt,radius=1pt},
  loop above/.append style={every loop/.append style={out=120, in=60, looseness=6}},
  loop below/.append style={every loop/.append style={out=300, in=240, looseness=6}},
  loop left/.append style={every loop/.append style={out=210, in=150, looseness=6}},
  loop right/.append style={every loop/.append style={out=30, in=330, looseness=6}},
  accepting arc/.style={dashed},
  marked/.style={
    dashed,
    opacity=0.3
  },
  marked on/.style={alt=#1{marked}{}},
}

\newcommand{\F}{\mathbf{F}}
\newcommand{\G}{\mathbf{G}}
\newcommand{\U}{\mathbf{U}}
\newcommand{\X}{\mathbf{X}}

\newcommand{\M}{\mathcal{M}}
\newcommand{\Mhat}{\widehat{\mathcal{M}}}

\tikzstyle{cs-robot-field-regular}=[inner sep=0pt,minimum width=1.0cm,minimum height=1.0cm,draw]
\tikzstyle{cs-robot-field-charge}=[inner sep=0pt,minimum width=1.0cm,minimum height=1.0cm,draw,fill=green!20!white]
\tikzstyle{cs-robot-field-danger}=[inner sep=0pt,minimum width=1.0cm,minimum height=1.0cm,draw,fill=red!20!white]

\newcommand{\csRobReg}[3]{\node[cs-robot-field-regular] at (#1,#2) {\huge\textbf{#3}};}

\newcommand{\csRobDan}[3]{\node[cs-robot-field-danger] at (#1,#2) {\huge\textbf{#3}};}

\title{A PAC Learning Algorithm for LTL and Omega-Regular Objectives in MDPs}

\author {
    Mateo Perez,
    Fabio Somenzi,
    Ashutosh Trivedi
}
\affiliations {
    University of Colorado Boulder
}

\begin{document}

\maketitle

\begin{abstract}
    Linear temporal logic (LTL) and $\omega$-regular objectives---a superset of LTL---have seen recent use as a way to express non-Markovian objectives in reinforcement learning. We introduce a model-based probably approximately correct (PAC) learning algorithm for $\omega$-regular objectives in Markov decision processes (MDPs). As part of the development of our algorithm, we introduce the $\varepsilon$-recurrence time: a measure of the speed at which a policy converges to the satisfaction of the $\omega$-regular objective in the limit. We prove that our algorithm only requires a polynomial number of samples in the relevant parameters, and perform experiments which confirm our theory.
\end{abstract}

\insert\footins{\phantom{A}\\\phantom{A}}

\section{Introduction}
Reinforcement learning (RL)~\cite{sutton2018reinforcement} is a sampling-based approach to learning a controller. Inspired by models of animal behavior, the RL agent interacts with the environment and receives feedback on its performance in terms of a numerical reward, that either reinforces or punishes certain behaviors. This learning approach has produced impressive results in recent years~\cite{mnih2015human, Silver16}. 
However, failure to precisely capture designer's intent in reward signals can lead to the agent learning unintended behavior~\cite{amodei2016concrete}. 
As a response, formal languages---in particular linear temporal logic (LTL) and $\omega$-regular languages---have been proposed to unambiguously capture learning objectives.
While these languages have enjoyed practical success~\cite{ hahn2019omega,bozkurt2020control}, their theoretical complexity is relatively underexplored.
In this paper we propose and study a model-based probably approximately correct RL algorithm for LTL and $\omega$-regular languages.

Probably approximately correct (PAC) learning~\cite{valiant1984theory} is a framework for formalizing guarantees of a learning algorithm: a user selects two parameters, $\varepsilon > 0$ and $\delta > 0$. A learning algorithm is then (efficient) PAC if it returns a solution that is $\varepsilon$ close to optimal with probability at least $1-\delta$ using a polynomial number of samples. In RL, many PAC learning algorithms have been proposed for both discounted and average reward~\cite{kakade2003sample, rmax}. These algorithms usually provide sample bounds in terms of the sizes of the state and action spaces of the Markov decision process (MDP) that describes the environment. Finite-horizon and discounted reward both have the property that small changes to the transition probabilities result in small changes to the value of the objective. This means that the sample complexity is independent of the transition probabilities of the MDP. However, infinite-horizon, undiscounted objectives, like average reward and the satisfaction of LTL properties, are sensitive to small changes in probabilities, and their sample complexity is dependent on some knowledge of the transition probabilities.  Hence, if only the number of state/action pairs is allowed, alongside $1/\varepsilon$ and $1/\delta$, as parameters, creating a PAC learning algorithm for undiscounted, infinite-horizon properties is not possible. Specifically for LTL, this has been observed by \citet{yang2021tractability} and \citet{alur2022framework}.  

\begin{example}[Intractability of LTL]
Figure~\ref{fig:robust} is an example adopted from \cite{alur2022framework} that shows the number of samples required to learn safety properties is dependent on some property of the transition structure. The objective in this example is to stay in the initial state $s_0$ forever. This can be specified with average reward (a reward of $1$ in $s_0$ and $0$ otherwise) and in LTL ($\varphi = \G s_0$). The transition from $s_0$ to $s_1$ under action $b$ must be observed in order to distinguish action $a$ from action $b$ and produce an $\varepsilon$-optimal policy for any $\varepsilon < 1$. The number of samples required to see this transition with high probability is affected by the value of $p$. Smaller values of $p$ means it takes longer for a policy's finite behavior to match its infinite behavior.

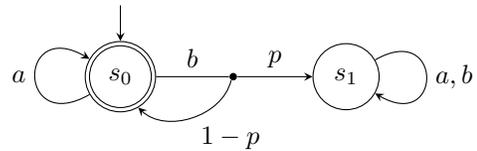
\begin{figure}[b]
    \centering
    \begin{tikzpicture}
        \node[state, initial above, accepting] (s0) {$s_0$};
        \node[circle, inner sep=1pt, fill=black] [right=1.5cm of s0] (p) {};
        \node[state] [right=1.5cm of p] (s1) {$s_1$};
        \path[-]
        (s0) edge node {$b$} (p)
        ;
        \path[->]
        (s0) edge [loop left] node {$a$} ()
        (p) edge [bend left=60] node {$1-p$} (s0)
        (p) edge node {$p$} (s1)
        (s1) edge [loop right] node {$a,b$} ()
        ;
    \end{tikzpicture}
    \caption{Example adopted from \cite{alur2022framework}. The objective is to remain in $s_0$ forever.}
    \label{fig:robust}
\end{figure}
\end{example}

This non-PAC-learnability may motivate using discounted versions of LTL \cite{littman2017environment,Alur23}, which, however, have significantly different semantics from the undiscounted logic.
One may argue instead that the complexity of the dynamics of an MDP is not entirely captured by the number of state-action pairs.  For example, for average reward, \citet{kearns2002near} use the $\varepsilon$-return mixing time, a measure of how fast the average reward is achieved in a particular system, for this purpose. They argue that in order to know the learning speed of an algorithm, one must know the speed at which the policy achieves the limit average reward. The R-MAX algorithm of \citet{rmax} also utilizes the $\varepsilon$-return mixing time.

The $\varepsilon$-return mixing time is defined based off of a given reward function, which we do not have in our context.
Therefore, we require an alternative notion. We propose the $\varepsilon$-recurrence time as a way to reason about the speed at which an $\omega$-regular objective is achieved.  
Informally, the $\varepsilon$-recurrence time is the expected time for a set of recurring states to be visited twice.
In Figure~\!\ref{fig:robust}, the $\varepsilon$-recurrence time increases when $p$ decreases.
We will show that this additional parameter is sufficient for defining a PAC algorithm for $\omega$-regular objectives.

\paragraph{Contributions.} We introduce a model-based PAC learning algorithm for LTL and $\omega$-regular objectives in Markov decision processes. For our algorithm, we introduce the $\varepsilon$-recurrence time: a measure of the speed at which a policy converges to the satisfaction of the $\omega$-regular objective in the limit. We show that the number of samples required by our algorithm is polynomial in the relevant input parameters. Our algorithm only requires the ability to sample trajectories of the system, and does not require prior knowledge of the exact graph structure of the MDP. 
Finally, we demonstrate the practicality of our algorithm on a set of case studies. 

\paragraph{Related work.} A PAC learning algorithm for LTL was introduced by~\citet{fu2014probably} that uses sampled trajectories to learn, but requires knowledge of the graph structure of the MDP, i.e., which transitions occur with nonzero probability. 
\citet{brazdil2014verification} propose an algorithm with PAC guarantees for unbounded reachability by using the minimum nonzero transition probability, and describe how to extend their method to LTL. 
\citet{ashok2019pac} utilize the minimum nonzero transition probability to develop an anytime statistical model-checking algorithm for unbounded reachability. Although they do not discuss it, in principle their method can extended to LTL similarly.
\citet{voloshin2022policy} provide an algorithm with PAC guarantees for LTL which assumes access to a generative model of the system.

\citet{daca2017faster} describe a PAC algorithm capable of checking the satisfaction of LTL on a Markov chain. They observe that ``some information about the Markov chain is necessary for providing statistical guarantees.'' 
The aforementioned works of~\citet{alur2022framework} and \citet{yang2021tractability} formalize this observation.

\citet{Alur23} study model-free RL for \emph{discounted} LTL, while we do not assume discounting.  \citet{HahnAtva22Impossibility} show that Rabin automata are unsuitable for model-free RL of $\omega$-regular objectives.  We can use Rabin automata because our algorithm is model-based.

\section{Preliminaries}
A Markov decision process (MDP) is a tuple $\M = (S, A, P, s_0)$ where $S$ is the set of states, $A$ is the set of actions, $P: S \times A \times S \to [0,1]$ is the transition function, and $s_0 \in S$ is the initial state. A run of an MDP is an infinite sequence $s_0, a_0, s_1, a_1, \ldots$ such that $P(s_i, a_i, s_{i+1}) > 0$ for all $i \ge 0$. A Markov chain $M = (S, P, s_0)$ is an MDP where the set of actions is singleton, i.e. $S$ is the set of states, $P: S \times S \to [0,1]$ is the transition function, and $s_0 \in S$ is the initial state. A bottom strongly connected component (BSCC) of a Markov chain is a bottom strongly connected component of the graph formed by the positive probability edges of the Markov chain. Equivalently, a BSCC of a Markov chain is a set of states $B \subseteq S$ where for all $s, s' \in B$ the probability of reaching $s'$ from $s$ is positive and the probability of reaching a state $s'' \in S\ \backslash\ B$ is zero. 
A policy is a recipe for selecting actions. A policy is positional if it is of the form $\pi : S \to A$. A policy $\pi$ induces a probability distribution over runs. We denote the probabilities under this distribution by $\Pr_\pi^\M(\cdot)$.

Let $AP$ be the set of atomic propositions. An LTL formula has the following grammar
$$\varphi := \top \mid b \in AP \mid \neg \varphi \mid \varphi \lor \varphi \mid \X \varphi \mid \varphi \U \varphi$$

We write $\bot := \neg \top$, $\varphi \land \varphi := \neg (\neg\varphi \lor \neg\varphi)$, the \emph{finally} operator as $\F \varphi := \top \U \varphi$, and the \emph{globally} operator $\G \varphi := \neg \F \neg \varphi$. For a formula $\varphi$ and an infinite word $\sigma = \sigma_0 \sigma_1 \ldots \in (2^{AP})^\omega$ we write $w \vDash \varphi$ to denote that $\sigma$ satisfies $\varphi$. We write $\sigma_{i:\infty} = \sigma_{i} \sigma_{i+1} \ldots$ for the substring of $\sigma$ starting at position $i$. The semantics of LTL are defined as
\[
\begin{aligned}
    &\sigma \vDash \top && \\
    &\sigma \vDash a && \text{iff\ } a \in \sigma_0 \\
    &\sigma \vDash \neg \varphi && \text{iff\ } \sigma \nvDash \varphi \\
    &\sigma \vDash \varphi_1 \lor \varphi_2 && \text{iff\ } \sigma \vDash \varphi_1 \text{\ or\ } \sigma \vDash \varphi_2 \\
    &\sigma \vDash \X \varphi && \text{iff\ } \sigma \vDash \sigma_{1:\infty} \\
    &\sigma \vDash \varphi_1 \U \varphi_2 && \text{iff\ } \exists j \ge 0 \text{\ s.t.\ } \sigma_{j:\infty} \vDash \varphi_2 \\
    & &&\qquad\text{and\ } \sigma_{i:\infty} \vDash \varphi_1, \forall\ 0 \le i < j\enspace.
\end{aligned}
\]

Omega-regular languages are a generalization of regular languages for infinite strings. Like regular languages are accepted by finite automata, $\omega$-regular languages are accepted by $\omega$-automata. An $\omega$-automaton is a tuple $\mathcal{A} = (Q, \Sigma, \delta, q_0, F)$ where $Q$ is a finite set of states, $\Sigma$ is the input alphabet, $\delta: Q \times \Sigma \to 2^Q$ is the (nondeterministic) transition function, $q_0 \in Q$ is an initial state, and $F$ is an acceptance condition over states. The B\"uchi acceptance condition is $F \subseteq Q$, a subset of accepting states. A B\"uchi automaton accepts an infinite word $\sigma$ if there exists a run in $\mathcal{A}$ that visits accepting states infinitely often. We denote the acceptance of an infinite word $\sigma$ by $\mathcal{A}$ as $\sigma \vDash \mathcal{A}$. It is well known that LTL expresses a subset of the $\omega$-regular languages. There exists many translations from LTL to $\omega$-automata~\cite{spot}.

Let $\M = (S, A, P, s_0, AP, L)$ be an MDP equipped with atomic propositions $AP$ and a labeling function $L: S \to 2^{AP}$, and let $\mathcal{A} = (Q, \Sigma, \delta, q_0, F)$ be an $\omega$-automaton. The probability of satisfaction of $\mathcal{A}$ under a policy $\pi$ in $\M$ is $p^\M_\pi = \Pr_{\pi}^\M(L(s_0)L(s_1)\ldots \vDash \mathcal{A})$. The optimal probability of satisfaction is $p^\M = \sup_{\pi} p^\M_\pi$. If $p^\M_\pi = p^\M$ then we say that $\pi$ is optimal. For some $\varepsilon > 0$, if $p_\pi^\M \ge p^\M - \varepsilon$ then $\pi$ is $\varepsilon$-optimal.

One can form the product MDP $\M^\times = (S^\times, A^\times, P^\times, (s_0, q_0), F)$ by synchronizing the states of $\M$ and $\mathcal{A}$, i.e. $S^\times = S \times Q$, $A^\times = A \times Q$, and $P^\times((s,q), (a,q), (s',q')) = P(s,a,s')$ if $q' \in \delta(q, L(s))$ and $P^\times((s,q), (a,q), (s',q')) = 0$ otherwise. Note that we can form the product on-the-fly, i.e., we can sample trajectories from $\M^\times$ by simply updating the states of $\M$ and $\mathcal{A}$ separately, and concatenating them at the end. Note that the nondeterminism in the automaton is resolved as actions in the product. If $\mathcal{A}$ is good-for-MDPs (GFM)~\cite{hahn2020good}, then the optimal probability of satisfaction in the product $\M^\times$ is the same as the optimal probability of satisfaction of $\mathcal{A}$ in $\M$. Deterministic $\omega$-regular automata are always GFM. One can lift the policy computed in $\M^\times$ to a memoryful policy in $\M$ that uses $\mathcal{A}$ as memory. We note that for the popular acceptance conditions B\"uchi, parity, and Rabin, computing optimal policies can be done in polynomial time in the size of $\M^\times$~.
We refer the reader to~\cite{baier2008principles} for further details. 

\section{Main Results}
For generality, we examine the setting in which we are given an MDP $\M = (S, A, P, s_0, F)$ equipped with an acceptance condition $F$ on states. To avoid unneeded complexity, we assume that the acceptance condition $F$ is such that there are positional optimal policies. This is true for the popular acceptance conditions B\"uchi, parity, and Rabin.\footnote{The results that follow can extended to a Muller condition by replacing instances of ``positional'' with ``deterministic finite memory with memory $N$'' where $N$ is a property dependent constant.}

Recall that if one begins with an MDP $\M = (S, A, P, s_0, AP, L)$ equipped with atomic propositions $AP$ and a labeling function $L: S \to 2^{AP}$, and a suitable GFM $\omega$-automaton---which may have been constructed from an LTL formula---then one can form the product MDP. The policy produced on the product MDP can be lifted to a memoryful strategy on the original MDP with the same guarantees. Additionally, recall that the product MDP construction can be done on-the-fly for unknown MDPs. Thus, our problem formulation is general enough to capture producing policies for LTL and $\omega$-regular objectives on MDPs with its states labeled by atomic propositions.

We begin by defining the $\varepsilon$-recurrence time in Markov chains.

\begin{definition}
    \label{def:rtime_MCs}
    The $\varepsilon$-recurrence time in a Markov chain $M = (S, P, s_0)$ is the smallest time $T$ such that with probability at least $1-\varepsilon$ a trajectory of length $T$ starting from the initial state $s_0$ visits every state in some BSCC twice.
\end{definition}

Intuitively, the $\varepsilon$-recurrence time $T$ is the time needed so that the recurrent behavior of a trajectory of length $T$ matches the recurrent behavior of an infinite extension of that trajectory with probability at least $1-\varepsilon$.
It may be that the $\varepsilon$-recurrence time is unknown, but other parameters, like the minimum positive transition probability $p_{\min}$ are known. In such a case, the following lemma provides an upper bound.
\begin{lemma}
    \label{lem:upper-bound}
    Let $M = (S, P, s_0)$ be a Markov chain and $p_{\min} = \min_{\{s,s' | P(s,s') > 0\}} P(s,s')$ be the minimum positive transition probability in $M$. Then the $\varepsilon$-recurrence time $T$ satisfies $T \le 2|S| \frac{\log(\varepsilon/2)}{\log(1-p_{\min}^{|S|})}$ .
\end{lemma}
\begin{proof}
    In the worst case, every state in the Markov chain must be seen at least twice and visiting every state in the Markov chain requires taking a path of length $|S|$ that occurs with probability $p_{\min}^{|S|}$. Attempting this path $k$ times, the probability of succeeding at least once is $1 - (1-p_{\min}^{|S|})^k$. If $k \ge \frac{\log(\varepsilon/2)}{\log(1-p_{\min}^{|S|})}$ then $1 - (1-p_{\min}^{|S|})^k \ge 1 - \varepsilon/2$. A lower bound on succeeding twice in $2k$ attemps is $(1 - \varepsilon) \le (1-\varepsilon/2)(1-\varepsilon/2)$. Finally, each of the $k$ attempts takes $|S|$ steps in the worst case to yield $T \le |S|2k = 2|S| \frac{\log(\varepsilon/2)}{\log(1-p_{\min}^{|S|})}$.
\end{proof}

We define the $\varepsilon$-recurrence time in MDPs so that we can reason about all positional policies in an MDP, as the optimal policies of interest are positional.
\begin{definition}
    \label{def:rtime_MDPs}
    The $\varepsilon$-recurrence time of an MDP $\M = (S, A, P, s_0)$ is the maximum $\varepsilon$-recurrence time amongst all Markov chains induced by positional policies in $\M$.
\end{definition}

The $\varepsilon$-recurrence time provides a measure of the speed at which finite trajectories converge to their infinite behavior, i.e., eventually dwell in a BSCC forever. To demonstrate the intuition behind the $\varepsilon$-recurrence time being sufficient to understand long term behavior from finite trajectories, we will sketch a simple model-free algorithm for estimating the probability of satisfaction $p$ in a Markov chain $M = (S, P, s_0, F)$. 
We will not use this algorithm when we consider MDPs, but it shows that the $\varepsilon$-recurrence time provides sufficient information to learn long term behavior.

Our algorithm samples $C$ trajectories of length $T$ from the initial state and observes the fraction of trajectories that are winning. 
As we will show, this algorithm has two sources of estimation error: the first since we sample finite length trajectories, and the second since we only sample finitely many times. To analyze the first type of error, we will utilize the definition of the $\varepsilon$-recurrence time. The second type follows from standard statistical results.

Fix $\varepsilon > 0$ and $\delta > 0$, and let $T$ be the $\varepsilon$-recurrence time in $M$. Let $p$ be the probability of satisfaction in $M$. 
Given a trajectory, we can form the trajectory graph by adding an edge from state $s$ to state $s'$ in the graph if a transition from $s$ to $s'$ was observed in the trajectory. Note that if we sampled infinite length trajectories, then the BSCC in the trajectory graph would correspond to a BSCC of the Markov chain. We identify a sampled trajectory as winning if the BSCC in the trajectory graph is winning. The proportion of infinite length trajectories identified as winning is exactly $p$. We now need to determine the error we accumulate from using trajectories of finite length $T$.

If we sample trajectories of length $T$ then the BSCC in the trajectory graph is also a BSCC in the Markov chain with probability at least $1-\varepsilon$, from the definition of the $\varepsilon$-recurrence time. This means that at least $1-\varepsilon$ of the trajectories are identified as if we had access to an oracle. Thus, our first type of error is at most $\varepsilon$.

Let $\widehat p$ be the proportion of trajectories of length $T$ that are identified as winning in expectation. We have that $|\widehat p - p| \le \varepsilon$.
Sampling trajectories of length $T$ thus gives us a coin biased by $\widehat p$ to toss. For the second type of error, we can give a statistical guarantee on estimating the weight of this coin from finite samples within some bound $\varepsilon' > 0$ of its true value. By using Hoeffding's inequality, we get that by sampling $C$ trajectories
\begin{align*}
    P(|\widehat p - \mathbb{E}[\widehat p]| \le \varepsilon') &\ge 1 - 2 \exp(-2\varepsilon'^2 C) \\
    P(|\widehat p - p| \le \varepsilon' + \varepsilon) &\ge 1 - 2\exp(-2\varepsilon'^2 C)
\end{align*}

For simplicity, we can set $\varepsilon' = \varepsilon$, and then select $C \ge \frac{-\ln(\delta / 2)}{2\varepsilon^2}$ so that $1 - 2\exp(-2\varepsilon^2 C) \ge 1 - \delta$. In summary, this algorithm is a model-free PAC algorithm for identifying the probability of satisfaction in Markov chains, i.e., it returns an estimated probability of satisfaction that is within $2\varepsilon$ of the true value $p$ with probability at least $1-\delta$ after polynomially-many samples.

We have shown that the $\varepsilon$-recurrence time is sufficient to reason about LTL and $\omega$-regular properties in Markov chains. We now turn our attention to MDPs, where we will develop a model-based PAC algorithm that uses the $\varepsilon$-recurrence time.

\subsection{The $\omega$-PAC Algorithm} 
For MDPs, we will develop a model-based PAC algorithm inspired by R-MAX~\cite{rmax} that utilizes the $\varepsilon$-recurrence time $T$ of $\M = (S, A, P, s_0, F)$. We call our algorithm $\omega$-PAC.

The general approach of our algorithm is to learn the transition probabilities of the MDP with high accuracy (within $\frac{\varepsilon}{|S|T}$ of their true values) and high confidence. We show that this implies that optimal policies on the learned MDP are $6\varepsilon$-optimal on the real MDP with high confidence (cf. Lemma~\ref{lem:approx} and Theorem~\ref{thm:correct}). To obtain our polynomial sample complexity results, we design our learned MDP to be \emph{optimistic}: one that provides an upper bound of the probability of satisfaction. This ensures that we continue to explore edges that we do not yet know with high accuracy sufficiently often (cf. Lemma~\ref{lem:explore} and Theorem~\ref{thm:complexity}).

Specifically, our approach keeps track of an estimate $\Mhat$ of the real system. State-action pairs in $\Mhat$ are kept in two categories: \emph{known} and \emph{unknown}. Known edges are edges we have sampled at least $k$ times, while unknown edges we have sampled less than $k$ times. Intuitively, we select $k$ so that known edges are edges we know with high accuracy (within $\frac{\varepsilon}{|S|T}$) and high confidence. For known edges, we use the observed transition distribution. For unknown edges, we set them as transitions to an accepting sink.\footnote{For B\"uchi, one can add the sink state to the accepting set. For parity, one can give this sink state an overriding winning priority (the largest odd priority for max odd semantics). For Rabin, one can add another pair that wins by visiting this sink state forever.} By setting the values of unknown edges optimistically high, an optimal positional policy $\pi$ in $\Mhat$ naturally explores the MDP.
The algorithm computes an optimal positional policy $\pi$ in $\Mhat$, samples trajectories of length $T$ from $s_0$ with $\pi$, and then updates $\Mhat$ from these samples.
When all edges that $\pi$ can visit in $T$ steps in $\Mhat$ are known, the algorithm stops and returns $\pi$. 

\begin{algorithm}[tb]
\caption{$\omega$-PAC}
\label{alg:algorithm}
\textbf{Input}: $|S|$, $|A|$, $T$, $\frac{1}{\varepsilon}$, $\frac{1}{\delta}$, and threshold $k > 0$ \\
\textbf{Output}: $6\varepsilon$-optimal policy $\pi$
\begin{algorithmic}[1]
\STATE Initialize $\Mhat$, policy $\pi$, and visit counts $c$
\WHILE{$\Mhat$ not known}
\STATE Compute optimal positional policy $\pi$ in $\Mhat$
\STATE Sample with $\pi$ for $T$ steps from initial state in $\M$
\STATE Update $\Mhat$ with threshold $k$
\ENDWHILE
\STATE \textbf{return} $\pi$
\end{algorithmic}
\end{algorithm}

We now present some more details of the $\omega$-PAC algorithm (see Algorithm~\ref{alg:algorithm}). We initialize the visit counts $c(s,a,s') \xleftarrow{} 0$ for all $s,s' \in S$ and $a \in A$, and $\pi$ to an arbitrary positional policy (Line $1$). Let $c(s,a) = \sum_{s' \in S} c(s,a,s')$. An edge is unknown if $c(s,a) < k$ and is known if $c(s,a) = k$. After sampling a trajectory $\tau \sim \{(s_0, a_0), \ldots, (s_{T-1}, a_{T-1})\}$ (Line $4$), for each $i \in \{0, 1, \ldots, T-1\}$ we update $c(s_i,a_i,s_{i+1}) \xleftarrow{} c(s_i,a_i,s_{i+1}) + 1$ only if $c(s_i,a_i) < k$. Once $c(s_i,a_i) = k$, we do not continue incrementing the visit counts. We use $\Mhat = (\widehat S, A, \widehat P, s_0, \widehat F)$ where $\widehat S = S \cup \{\mathtt{sink}\}$,
\[
\widehat P(s,a,s') =
\begin{cases} 
\mathds{1}_{s' = \mathtt{sink}} & c(s,a) < k \lor s = \mathtt{sink} \\
\frac{c(s,a,s')}{c(s,a)} & c(s,a) = k \land s \neq \mathtt{sink}
\end{cases}
\]
and
\[
\widehat F(s) =
\begin{cases} 
F(s) & s \neq \mathtt{sink} \\
\mathtt{accepting} & s = \mathtt{sink}
\end{cases}
\]
for all instances of $\Mhat$ (Lines $1$ and $5$), where $\mathds{1}_{s'=\mathtt{sink}}$ is the indicator function for $s'=\mathtt{sink}$. Note that these updates can be performed without knowing $S$, $A$, or $F$ apriori as we only update $\widehat P$ to something nontrivial for states that have been visited.

A naive stopping condition (Line $2$) would be to stop only when all edges are marked as known. Instead, we will use a more general condition, that all of the edges reachable in $T$ steps under $\pi$ are known. Formally, let $S_T \subseteq S$ be the set of states reachable in $T$ steps with positive probability under $\pi$ in $\Mhat$ from $s_0$. The condition on line $2$ holds if $c(s,a) = k$ for all $s \in S_T$ and $a \in A$.

We have presented $\omega$-PAC as an algorithm that returns a single policy $\pi$. The same algorithm can also be phrased as producing an infinite sequence of policies $\pi_i$ for all timesteps $i \ge 0$ where $\pi_i$ be the policy $\pi$ in the $\omega$-PAC learning loop after $i$ samples of the system have been taken. If $i$ is greater than the number of samples $\omega$-PAC takes, we define $\pi_i$ as the policy returned by $\omega$-PAC.

We show that for a selection of $k$ that is polynomial in the input parameters, the policy $\pi$ returned by the $\omega$-PAC algorithm is $6\varepsilon$-optimal with probability at least $1-\delta$ (Theorem~\ref{thm:correct}). We will also show our algorithm has a polynomial sample complexity, i.e., the policy $\pi$ while the algorithm is running is not $9\varepsilon$-optimal at most some polynomial number of times with probability at least $1-2\delta$ (Theorem~\ref{thm:complexity}). We will now introduce the machinery required to prove these results.

We begin by defining an $(\alpha, T)$-approximation. This is an approximation of an MDP where the probabilities of all transitions up to a depth $T$ are known within $\alpha$.

\begin{definition}
    \label{def:approx}
    An $(\alpha,T)$-approximation of an MDP $\M = (S, A, P, s_0, F)$ is an MDP $\M' = (S, A, P', s_0, F)$ such that for all $s, s' \in S_T$ and $a \in A$, $|P(s, a, s') - P'(s, a, s')| \le \alpha$ and, if $P(s, a, s') = 0$, then $P'(s, a, s') = 0$, where $S_T \subseteq S$ are the states reachable with positive probability in $T$ steps from $s_0$ under some strategy.
\end{definition}

Note that an $(\alpha,T)$-approximation of an MDP can be obtained by averaging samples of observed trajectories of length $T$ to produce an estimate of the transition probabilities $P'(s, a, s')$. If $P(s, a, s') = 0$ then that transition is never observed, so $P'(s, a, s') = 0$. Additionally, enough samples will yield $|P(s, a, s') - P'(s, a, s')| \le \alpha$ with high probability. We show this explicitly in the following lemma.

\begin{lemma}
    \label{lem:sample}
    Let $0 < \delta < 1$, $\alpha > 0$, and $M = (S, P, s_0, F)$ be a Markov chain. Let $\widehat P(s,s') = \frac{c(s,s')}{k}$ where $c(s,s')$ is the number of observed transitions from $s$ to $s'$ obtained after sampling transitions from a state $s$, $k \ge \big\lceil \frac{-\ln(\delta / 2)}{2 \alpha^2} \big\rceil$ times. Then with probability at least $1 - \delta$, $|\widehat P(s,s') - P(s,s')| \le \alpha$ and $\widehat P(s,s') = 0$ if $P(s,s') = 0$ for all $s' \in S$.
\end{lemma}
\begin{proof}
    Fix $s' \in S$. Since $P(s, s') = 0$ implies that $c(s,s') = 0$ and thus $\widehat P(s,s') = 0$, all we need to show is that $|\widehat P(s,s') - P(s,s')| \le \alpha$ with probability at least $1-\delta$. We apply Hoeffding's inequality to get
    \begin{align*}
        \Pr(|\hat P(s,s') - P(s,s')| \le \alpha) &\ge 1 - 2 \exp(-2\alpha^2 k) .
    \end{align*}
    Substituting, we get that 
    \begin{align*}
        1 - 2\exp(-2\alpha^2 k) &\ge 1 - \delta . \qedhere
    \end{align*}
\end{proof}

This lemma is helpful for giving a bound on the number of samples required to learn an $(\alpha, T)$-approximation. In order to determine the appropriate $\alpha$ to select, we'd like to give a bound on the change in the probability of satisfaction between an MDP $\M$ and its $(\alpha,T)$-approximation $\M'$. To provide such a bound, we use the following result.

\begin{lemma}
    \label{lem:t-reach}
    Let $M = (S, P, s_0, F)$ be a Markov chain, $M' = (S, P', s_0, F)$ be an $(\alpha,T)$-approximation of $M$, $N = |S|$ denote the size of the state space. If the probability to reach $s' \in S$ from $s \in S$ in at most $T$ steps in $M$ is $p$, then the probability to reach $s'$ from $s$ in at most $T$ steps in $M'$ is at least $p - \alpha N T$.
\end{lemma}

\begin{proof}
    Let $R_i$ and $R_i'$ be the events that we reached $s'$ from $s$ in at most $i$ steps in  $M$ and $M'$ respectively. 
    We'd like to show that $\Pr(R_i') \ge \Pr(R_i) - \alpha N i$ for all $i \ge 0$. We show this by induction.
    For the base case, it is clear that $\Pr(R_0') = \Pr(R_0)$. 
    
    For convenience, we define $p_{i} = \Pr(R_i)$ and $p_{i}' = \Pr(R_i')$. We also define $p_{i|i-1} = \Pr(R_i | \neg R_{i-1})$ and $p_{i|i-1}' = \Pr(R_i' | \neg R_{i-1}')$.
    Since there are $N$ total transitions, the worst case reduction in the single step transition probabilities between states is at most $\alpha N$. Thus, $p_{i|i-1}' \ge p_{i|i-1} - \alpha N$.
    For the inductive step, we can write for $i > 0$ 
    \begin{align*}
        \Pr(R_i') &= \Pr(R_i' | \neg R_{i-1}') \Pr(\neg R_{i-1}') + \Pr(R_{i-1}') \\
        &= p_{i|i-1}' (1 - p_{i-1}') + p_{i-1}' \\
        &\ge p_{i|i-1}' (1 - p_{i-1}) + (p_{i-1} - \alpha N (i-1)) \\
        &\ge (p_{i|i-1} - \alpha N) (1 - p_{i-1}) + (p_{i-1} - \alpha N (i{-}1)) \\
        &\ge p_{i|i-1} (1 - p_{i-1}) + p_{i-1} - \alpha N i \\
        &= \Pr(R_i) - \alpha N i \tag*{\qedhere}
    \end{align*} 
\end{proof}

We are now ready to bound the difference in the probability of satisfaction between a Markov chain $M$ and its $(\alpha,T)$-approximation $M'$.

\begin{lemma}
    \label{lem:approx}
    Let $M = (S, P, s_0, F)$ be a Markov chain, $N = |S|$ denote the size of the state space, $\varepsilon > 0$, and $T$ be the $\varepsilon$-recurrence time in $M$. Let $M' = (S, P', s_0, F)$ be an $(\frac{\varepsilon}{NT},T)$-approximation of $M$, and $T'$ be the $2\varepsilon$-recurrence time in $M'$. Let $p$ and $p'$ be the probability of satisfaction from $s_0$ in $M$ and $M'$, respectively. Then, 
    \begin{enumerate}
        \item $T' \le T$
        \item $|p' - p| \le 3\varepsilon$ .
    \end{enumerate}
\end{lemma}

\begin{proof}
    We begin by defining the unrolling of a Markov chain and the associated set of lasso states. The unrolling of a Markov chain $M = (S, P, s_0, F)$ is a Markov chain  $M_x = (S_x, P_x, (s_0, \mathbf{0}), F_x)$ that has the same dynamics as $M$, but keeps track of the visitation counts of each state.
    The set of lasso states $L$ of $M_x$ is the set of states such that there exists a BSCC $B$ in $M$ such that all the visitation counts are greater than or equal to $2$ for all $s \in B$. Given a state $s \in L$, we define $L^{-1}(s) = B$ as the function that returns the BSCC $B$ in $M$ corresponding to that state in $M_x$.

    Consider the unrolled Markov chains $M_x$ and $M_x'$, and their lasso states $L$ and $L'$, for $M$ and $M'$, respectively. The probability of visiting a state $s \in L$ from $(s_0, \mathbf{0})$ in $M_x$ in $T$ steps is at least $1 - \varepsilon$, by definition. By Lemma~\ref{lem:t-reach}, the probability of visiting a state $s \in L$ from $(s_0, \mathbf{0})$ in $M_x'$ in $T$ steps is at least $1 - 2\varepsilon$. Let $X \subseteq L$ be the set of states in $L$ that are reached with positive probability in $T$ steps in $M_x'$, and let $\mathcal{B} = \{L^{-1}(x) : x \in X\}$. For each $B \in \mathcal{B}$, all states $s \in B$ can reach each other in $M'$ with positive probability in $T$ steps by definition, and thus are part of the same SCC in $M'$. Since $P'(s,s') = 0$ if $P(s,s') = 0$, these states form a BSCC in $M'$.
    
    In summary, every BSCC $B \in \mathcal{B}$ is a BSCC in $M$ and $M'$, and the probability of reaching a state $s \in B$ in $T$ steps from $s_0$ in $M'$ is at least $1-2\varepsilon$. Thus, $T$ is an upper bound on the $2\varepsilon$-recurrence time in $M'$, proving part 1. Finally, let $p_\mathcal{B}$ and $p_\mathcal{B}'$ be the probability of reaching a winning BSCC $B \in \mathcal{B}$ in $T$ steps from $s_0$ in $M$ and $M'$, respectively. Then,
    \begin{align}
        |p' - p| &\le |p_\mathcal{B}' - p_\mathcal{B}| + 2\varepsilon \\
        &\le \varepsilon + 2\varepsilon = 3\varepsilon
    \end{align}
    where (1) follows from the fact that the BSCCs in $\mathcal{B}$ are reached with probability at least $1 - 2\varepsilon$ in $M'$, and (2) follows from applying Lemma~\ref{lem:t-reach}. This proves part 2.
\end{proof}

Since Definition~\ref{def:rtime_MDPs} is concerned with positional policies, and optimal policies are positional, Lemma~\ref{lem:approx} applies directly to an MDP $\M$ and its $(\frac{\varepsilon}{NT}, T)$-approximation $\M'$, producing the same bounds. 
This motivates selecting the number of samples $k$ to mark an edge as known in the $\omega$-PAC algorithm to be such that we are highly confident that we have an $(\frac{\varepsilon}{NT}, T)$-approximation of $\M$. We can use Lemma~\ref{lem:sample} to select such a $k$. We are now ready to show the correctness of the $\omega$-PAC algorithm under the appropriate selection of $k$.

\begin{theorem}[Correctness]
    \label{thm:correct}
    Let $0 < \delta < 1$, $\varepsilon > 0$, $\M = (S, A, P, s_0, F)$ be an MDP, $N = |S|$ denote the size of the state space, $K = |A|$ denote the size of the action space, and $T$ be the $\varepsilon$-recurrence time of $\M$. Let $\varepsilon' = \frac{\varepsilon}{NT}$ and $\delta' = \frac{\delta}{NK}$. For $k = \big\lceil \frac{-\ln(\delta' / 2)}{2 \varepsilon'^2} \big\rceil$, the policy $\pi$ returned by $\omega$-PAC is $6\varepsilon$-optimal with probability at least $1-\delta$.
\end{theorem}

\begin{proof}
    Let $\M'$ be some $(\frac{\varepsilon}{NT}, T)$-approximation of $\M$. Let $\sigma$ be an optimal positional policy in $\M'$. Let $p$ be the optimal probability of satisfaction in $\M$, and let $p_\sigma$ and $p_\sigma'$ be the probability of satisfaction in $\M$ and $\M'$ under $\sigma$, respectively. By Lemma~\ref{lem:approx}, we have that
    \begin{align*}
        |p - p_\sigma| &\le |p - p_\sigma'| + |p_\sigma' - p_\sigma| \\
        &\le 3\varepsilon + 3\varepsilon = 6\varepsilon .
    \end{align*}
    Thus, all we need to show is that with probability at least $1-\delta$ there exists an $(\frac{\varepsilon}{NT},T)$-approximation $\M'$ of $\M$ such that $\pi$ is optimal in $\M'$.

    Let $\widehat \M$ denote the optimistic MDP when $\omega$-PAC terminates. We say that a state-action pair $s \in S$, $a \in A$ in $\widehat \M$ is $\alpha$-accurate if for all $s' \in S$, $|\widehat P(s, a, s') - P(s, a, s')| \le \alpha$ and if $P(s, a, s') = 0$ then $\widehat P(s, a, s') = 0$. By Lemma~\ref{lem:sample}, a state-action pair marked as known is $\frac{\varepsilon}{NT}$-accurate with probability at least $1-\delta'$. Since there are $NK$ total state-action pairs, the probability that all state-action pairs marked as known are $\frac{\varepsilon}{NT}$-accurate is at least $(1-\delta')^{NK} \ge 1 - \delta$. 
    Let $\M' = (S, A, \widehat P', s_0, F')$ be an MDP such that the transition probabilities for all known state-action pairs are identical to $\widehat \M$, are $\frac{\varepsilon}{NT}$-accurate for unknown state-action pairs that are reachable in $T$ steps from $s_0$ with positive probability under some strategy, and are accepting sinks otherwise. With probability at least $1-\delta$, $\M'$ is a $(\frac{\varepsilon}{NT},T)$-approximation of $\M$. Finally, note that the probability of satisfaction in $\widehat \M$ and $\M'$ under $\pi$ is the same since $\omega$-PAC terminates when $\pi$ only visits known state-action pairs in $T$ steps. Therefore, since the optimal probability of satisfaction $\widehat p$ in $\widehat \M$ is an upper bound on the probability of satisfaction in $\M'$, by the construction of $\widehat \M$, $\pi$ is optimal in $\M'$.
\end{proof}

Note that $k = \Tilde O(|S|^2 T^2 / \varepsilon^2)$ selected in the previous theorem is bounded by a polynomial in the input parameters.
For Theorem~\ref{thm:correct}, we assume we run the algorithm until termination, which occurs with probability $1$: if it has not terminated, $\pi$ visits an unknown state-action pair with positive probability in $T$ steps, and there can only be $k |S| |A|$ such visits before all state-action pairs are marked as known.
We now show sample complexity bounds for the $\omega$-PAC algorithm by showing that the number of timesteps that $\pi$ is not $9\varepsilon$-optimal is bounded by a polynomial in $|S|$, $|A|$, $T$, $\frac{1}{\varepsilon}$, and $\frac{1}{\delta}$ with probability at least $1-2\delta$. For such a sample complexity result, we need to reason about how often unknown state-action pairs are visited. We show this in the following lemma.

\begin{lemma}
    \label{lem:explore}
    Let $\M = (S, A, P, s_0, F)$ be an MDP. Let $\widehat \M = (S \cup \{\mathtt{sink}\}, A, \widehat P, s_0, \widehat F)$ be identical to $\M$ except some arbitrary set $U$ of state-action of pairs are converted into transitions to the accepting sink. Let $\pi$ be a positional optimal policy in $\widehat \M$, $\alpha > 0$, $\varepsilon > 0$, and $T$ be the $\varepsilon$-recurrence time of $\M$. Then at least one of the following holds:
    \begin{enumerate}
        \item $\pi$ is $\alpha$-optimal from $s_0$ in $\M$, or
        \item a trajectory in $\widehat \M$ of length $T$ from $s_0$ under $\pi$ visits a state-action pair in $U$ with probability at least $\alpha - \varepsilon$.
    \end{enumerate}
\end{lemma}
\begin{proof}
    To prove this lemma, it is sufficient to show that if $\pi$ is not $\alpha$-optimal from $s_0$ in $\M$, then a trajectory in $\widehat \M$ of length $T$ from $s_0$ under $\pi$ visits a state-action pair in $U$ with probability at least $\alpha - \varepsilon$.
    Let $p_\pi$ and $\widehat p_\pi$ be the probability of satisfaction that $\pi$ obtains from $s_0$ in $\M$ and $\widehat \M$, respectively. Let $p$ be the maximum probability of satisfaction in $\M$. We begin by noting that $\widehat p_\pi \ge p$ by the construction of $\widehat \M$. If $\pi$ is not $\alpha$-optimal from $s_0$ in $M$, this means that $p  - p_\pi \ge \alpha$, which implies $\widehat p_\pi - p_\pi \ge \alpha$. As the values $\widehat p_\pi$ and $p_\pi$ only differ due to $\pi$ reaching state-action pairs in $U$ in $\widehat \M$, this means that $\pi$ must reach a state-action pair in $U$ in $\widehat \M$ from $s_0$ with probability at least $\alpha$.

    Finally, note that $T+1$ is an upper bound on the $\varepsilon$-recurrence time in $\widehat \M$. This is because any policy $\pi$ in $\widehat \M$ that takes a state-action pair in $U$ will visit the BSCC formed by the sink after one additional timestep. For reasoning about reaching a state-action pair in $U$ once, this additional timestep due to the sink state has no effect. Thus, if the probability to reach a state-action pair in $U$ in $\widehat \M$ from $s_0$ under $\pi$ is at least $\alpha$, it must be at least $\alpha - \varepsilon$ in $T$ steps.
\end{proof}

We are now able to show the sample complexity of our algorithm. Note that the bound in Theorem~\ref{thm:complexity} on the number of samples $C = \Tilde O(|S|^3 |A| T^3 / \varepsilon^4)$ is bounded by a polynomial in $|S|$, $|A|$, $T$, $\frac{1}{\varepsilon}$, and $\frac{1}{\delta}$. 
\begin{theorem}[Sample Complexity]
    \label{thm:complexity}
    Let $0 < \delta < 1$, $\varepsilon > 0$, $\M = (S, A, P, s_0, F)$ be an MDP, $N = |S|$ denote the size of the state space, $K = |A|$ denote the size of the action, and $T$ be the $\varepsilon$-recurrence time of $\M$. Let $\varepsilon' = \frac{\varepsilon}{NT}$ and $\delta' = \frac{\delta}{NK}$. Let $\pi_i$ be an infinite sequence of policies produced by $\omega$-PAC. For $k = \big\lceil \frac{-\ln(\delta' / 2)}{2 \varepsilon'^2} \big\rceil$, $\pi_i$ is not $9\varepsilon$-optimal for at most $C = T \big\lceil \max\left( \frac{kNK}{\varepsilon}, \frac{kNK - \ln(\delta)}{2\varepsilon^2} \right) \big\rceil$ steps with probability at least $1 - 2\delta$.
\end{theorem}
\begin{proof}
    From the proof of Theorem~\ref{thm:correct}, all of the state-action pairs marked as known in $\widehat \M$ are $\frac{\varepsilon}{NT}$-accurate with probability at least $1-\delta$. For ease of presentation, we will assume that this occurs and incorporate its probability at end of the proof.
    
    Let $\M'$ be an $(\frac{\varepsilon}{NT}, T)$-approximation of $\M$ that matches $\widehat \M$ for all of the state-action pairs marked as known at the end of training. By Lemma~\ref{lem:approx}, the $2\varepsilon$-recurrence time in $\M'$ is at most $T$. The maximum number of visits to unknown state-action pairs is $kNK$, since all $NK$ state-action pairs will be marked as known after this. By Lemma~\ref{lem:explore}, if the policy $\pi_i$ is not $3\varepsilon$-optimal in $\M'$, the algorithm will visit an unknown state-action pair with probability at least $\varepsilon$.  Let $m$ be the number of steps that $\pi_i$ is not $3\varepsilon$-optimal in $\M'$ over the course of training. We now show that $\Pr(m \le C) \ge 1 - \delta$. 
    Let $\mathcal{S}$ be the number of successes of a binary random variable that occurs with probability $\varepsilon$ sampled $C/T$ times. Since $kNK \le \varepsilon \frac{C}{T}$, we can apply Hoeffding's inequality to get that
    \begin{align*}
        \Pr(m \le C) &\ge \Pr(\mathcal{S} > kNK) \\
        &\ge 1 - \textstyle\exp(-\frac{2C}{T} (\varepsilon - \frac{kNKT}{C})^2) \\
        &\ge 1 - \delta .
    \end{align*}
    From the proof of Theorem~\ref{thm:correct}, since $\M'$ is an $(\frac{\varepsilon}{NT}, T)$-approximation, $\pi_i$ is $6\varepsilon$-optimal in $\M$. Recalling that we assumed that all state-action pairs in $\widehat \M$ are $\frac{\varepsilon}{NT}$-accurate, which occurs with probability $1-\delta$, we have that the overall probability of producing a $9\varepsilon$-optimal strategy is at least $(1-\delta)(1-\delta) \ge 1 - 2\delta$.
\end{proof}

\section{Experiments}

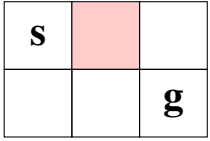
\begin{figure}
    \centering
    \begin{tikzpicture}[transform shape, scale=0.9]
        \csRobReg{0}{1}{s}
        \csRobDan{1}{1}{}
        \csRobReg{2}{1}{}

        \csRobReg{0}{0}{}
        \csRobReg{1}{0}{}
        \csRobReg{2}{0}{g}
    \end{tikzpicture}
    \caption{The \textbf{gridworld example}. The red cell is a trap. The goal is to visit the two states $s$ and $g$ repeatedly.}
    \label{fig:grid}
\end{figure}

We implemented $\omega$-PAC inside of the tool Mungojerrie~\cite{Tacas23MJ}.\footnote{Available at https://plv.colorado.edu/mungojerrie/omega-pac.}
Mungojerrie can compute optimal policies with respect to a parity automaton in MDPs and is written in C++. All experiments were run on a computer with an Intel i7-8750H processor and 16 GB of memory.

\paragraph{Gridworld example.} Figure~\ref{fig:grid} shows a gridworld example. In this example, the agent has four actions, north-east, north-west, south-east, and south-west. For a given direction, the agent moves in one of the corresponding cardinal directions with probability $0.4$, in the other corresponding cardinal direction with probability $0.4$, and does not move with probability $0.2$. If the agent would move into a wall, it does not move. In the trap state denoted in red, the agent becomes stuck and all actions cause the agent to not move. The property is to visit the states $s$ and $g$ infinitely often, which is expressible in LTL as $\varphi = \G\F s \land \G\F g$. We set $\varepsilon = 1/20$, and $\delta = 1/10$. The product contains $|S| = 12$ states, $|A| = 4$ actions, and has a $\varepsilon$-recurrence time of $T = 19$. Our implementation of the $\omega$-PAC algorithm takes approximately $40$ minutes to terminate on this example, under the parameter selection for $k$ suggested by Theorem~\ref{thm:correct}. We did not observe a run where the resulting policy produced was not optimal under this parameter selection, suggesting that the $k$ in Theorem~\ref{thm:correct} may be needlessly large in practice.

\paragraph{Chain example.} To investigate the effect of different values of $k$ on the performance of $\omega$-PAC, we examined a simple MDP consisting of a chain of states with two actions: one action continues, and the other goes to an accepting sink with probability $1/2^s$ for the $s^{th}$ state. In this example, $|S| = 8$, $|A| = 2$, $T  = 8$, $\varepsilon = 1/60$, and $\delta = 1/10$. Figure~\ref{fig:chain-plot} shows the distribution of probabilities of satisfaction of the policies produced by $\omega$-PAC for $20$ runs under different $k$, up to the $k$ used in Theorem~\ref{thm:correct}. We see that in practice, a small $k$ typically suffices, and that results of this example are in line with Theorem~\ref{thm:correct}.

\begin{figure}
    \centering
    \includegraphics[scale=0.5]{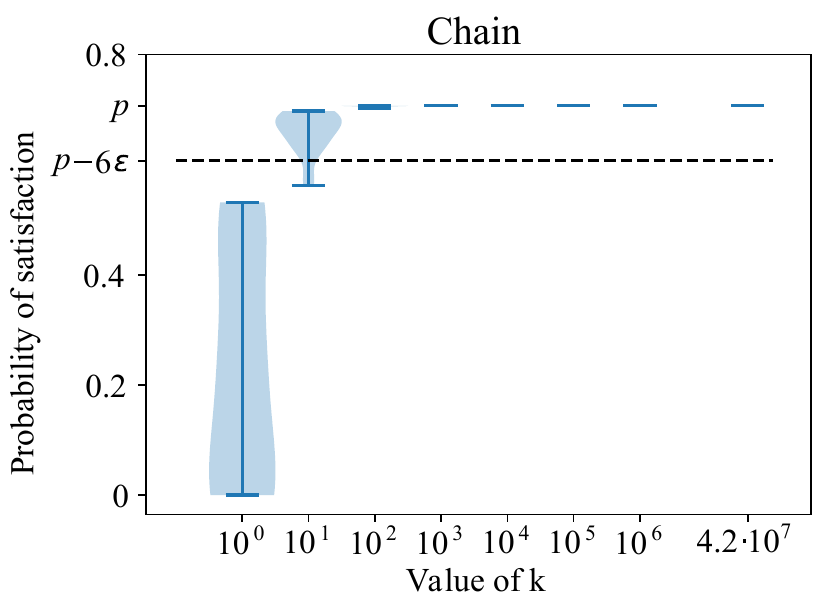}
    \caption{The distribution of the probability of satisfaction of learned policies for different values of $k$ for the \textbf{chain example}. We also show the optimal probability of satisfaction $p$ and the threshold for $6\varepsilon$-optimality.}
    \label{fig:chain-plot}
\end{figure}

\section{Conclusion}
We introduced $\omega$-PAC, a PAC learning algorithm for LTL and $\omega$-regular objectives in MDPs. For this algorithm, we introduced the notion of the $\varepsilon$-recurrence time. Intuitively, the $\varepsilon$-recurrence time measures the time it takes for finite trajectories to match the recurrent behavior of infinite trajectories with high probability. We proved that the $\omega$-PAC algorithm has a sample complexity that is polynomial in the relevant parameters, the size of the state space $|S|$, the size of the action space $|A|$, the $\varepsilon$-recurrence time $T$, $\frac{1}{\varepsilon}$, and $\frac{1}{\delta}$. Finally, we performed experiments with $\omega$-PAC that suggest that the bounds of our theory can be tightened as part of future work.

\section*{Acknowledgements}
We thank Mahmoud Salamati and Suguman Bansal for valuable discussions.
This work was supported in part by the NSF through grant CCF-2009022 and the NSF CAREER award CCF-2146563.

\bibliography{papers}

\begin{thebibliography}{24}
\providecommand{\natexlab}[1]{#1}

\bibitem[{Alur et~al.(2022)Alur, Bansal, Bastani, and
  Jothimurugan}]{alur2022framework}
Alur, R.; Bansal, S.; Bastani, O.; and Jothimurugan, K. 2022.
\newblock A framework for transforming specifications in reinforcement
  learning.
\newblock In \emph{Principles of Systems Design: Essays Dedicated to Thomas A.
  Henzinger on the Occasion of His 60th Birthday}, 604--624. Springer.

\bibitem[{Alur et~al.(2023)Alur, Bastani, Jothimurugan, Perez, Somenzi, and
  Trivedi}]{Alur23}
Alur, R.; Bastani, O.; Jothimurugan, K.; Perez, M.; Somenzi, F.; and Trivedi,
  A. 2023.
\newblock Policy Synthesis and Reinforcement Learning for Discounted {LTL}.
\newblock In \emph{Computer Aided Verification}, 415--435.

\bibitem[{Amodei et~al.(2016)Amodei, Olah, Steinhardt, Christiano, Schulman,
  and Man{\'e}}]{amodei2016concrete}
Amodei, D.; Olah, C.; Steinhardt, J.; Christiano, P.; Schulman, J.; and
  Man{\'e}, D. 2016.
\newblock Concrete problems in AI safety.
\newblock \emph{arXiv preprint arXiv:1606.06565}.

\bibitem[{Ashok, Kretinsky, and Weininger(2019)}]{ashok2019pac}
Ashok, P.; Kretinsky, J.; and Weininger, M. 2019.
\newblock {PAC} statistical model checking for Markov decision processes and
  stochastic games.
\newblock In \emph{International Conference on Computer Aided Verification},
  497--519. Springer.

\bibitem[{Baier and Katoen(2008)}]{baier2008principles}
Baier, C.; and Katoen, J.-P. 2008.
\newblock \emph{Principles of model checking}.
\newblock MIT press.

\bibitem[{Bozkurt et~al.(2020)Bozkurt, Wang, Zavlanos, and
  Pajic}]{bozkurt2020control}
Bozkurt, A.~K.; Wang, Y.; Zavlanos, M.~M.; and Pajic, M. 2020.
\newblock Control synthesis from linear temporal logic specifications using
  model-free reinforcement learning.
\newblock In \emph{2020 IEEE International Conference on Robotics and
  Automation (ICRA)}, 10349--10355. IEEE.

\bibitem[{Brafman and Tennenholtz(2003)}]{rmax}
Brafman, R.~I.; and Tennenholtz, M. 2003.
\newblock R-Max - a General Polynomial Time Algorithm for near-Optimal
  Reinforcement Learning.
\newblock \emph{Journal of Machine Learning Research}, 3.

\bibitem[{Br{\'a}zdil et~al.(2014)Br{\'a}zdil, Chatterjee, Chmelik, Forejt,
  K{\v{r}}et{\'\i}nsk{\`y}, Kwiatkowska, Parker, and
  Ujma}]{brazdil2014verification}
Br{\'a}zdil, T.; Chatterjee, K.; Chmelik, M.; Forejt, V.;
  K{\v{r}}et{\'\i}nsk{\`y}, J.; Kwiatkowska, M.; Parker, D.; and Ujma, M. 2014.
\newblock Verification of Markov decision processes using learning algorithms.
\newblock In \emph{Automated Technology for Verification and Analysis: 12th
  International Symposium, ATVA 2014, Sydney, NSW, Australia, November 3-7,
  2014, Proceedings 12}, 98--114. Springer.

\bibitem[{Daca et~al.(2017)Daca, Henzinger, Kretinsky, and
  Petrov}]{daca2017faster}
Daca, P.; Henzinger, T.~A.; Kretinsky, J.; and Petrov, T. 2017.
\newblock Faster statistical model checking for unbounded temporal properties.
\newblock \emph{ACM Transactions on Computational Logic (TOCL)}, 18(2): 1--25.

\bibitem[{Duret-Lutz et~al.(2016)Duret-Lutz, Lewkowicz, Fauchille, Michaud,
  Renault, and Xu}]{spot}
Duret-Lutz, A.; Lewkowicz, A.; Fauchille, A.; Michaud, T.; Renault, E.; and Xu,
  L. 2016.
\newblock Spot 2.0 --- a framework for {LTL} and $\omega$-automata
  manipulation.
\newblock In \emph{Proceedings of the 14th International Symposium on Automated
  Technology for Verification and Analysis (ATVA'16)}, volume 9938 of
  \emph{Lecture Notes in Computer Science}, 122--129. Springer.

\bibitem[{Fu and Topcu(2014)}]{fu2014probably}
Fu, J.; and Topcu, U. 2014.
\newblock Probably approximately correct {MDP} learning and control with
  temporal logic constraints.
\newblock \emph{arXiv preprint arXiv:1404.7073}.

\bibitem[{Hahn et~al.(2019)Hahn, Perez, Schewe, Somenzi, Trivedi, and
  Wojtczak}]{hahn2019omega}
Hahn, E.~M.; Perez, M.; Schewe, S.; Somenzi, F.; Trivedi, A.; and Wojtczak, D.
  2019.
\newblock Omega-regular objectives in model-free reinforcement learning.
\newblock In \emph{Tools and Algorithms for the Construction and Analysis of
  Systems: 25th International Conference, TACAS 2019, Held as Part of the
  European Joint Conferences on Theory and Practice of Software, ETAPS 2019,
  Prague, Czech Republic, April 6--11, 2019, Proceedings, Part I}, 395--412.
  Springer.

\bibitem[{Hahn et~al.(2020)Hahn, Perez, Schewe, Somenzi, Trivedi, and
  Wojtczak}]{hahn2020good}
Hahn, E.~M.; Perez, M.; Schewe, S.; Somenzi, F.; Trivedi, A.; and Wojtczak, D.
  2020.
\newblock Good-for-MDPs automata for probabilistic analysis and reinforcement
  learning.
\newblock In \emph{Tools and Algorithms for the Construction and Analysis of
  Systems: 26th International Conference, TACAS 2020, Held as Part of the
  European Joint Conferences on Theory and Practice of Software, ETAPS 2020,
  Dublin, Ireland, April 25--30, 2020, Proceedings, Part I}, 306--323.
  Springer.

\bibitem[{Hahn et~al.(2022)Hahn, Perez, Schewe, Somenzi, Trivedi, and
  Wojtczak}]{HahnAtva22Impossibility}
Hahn, E.~M.; Perez, M.; Schewe, S.; Somenzi, F.; Trivedi, A.; and Wojtczak, D.
  2022.
\newblock An Impossibility Result in Automata-Theoretic Reinforcement Learning.
\newblock In \emph{Automated Technology for Verification and Analysis}, volume
  13505 of \emph{Lecture Notes in Computer Science}, 42--57.

\bibitem[{Hahn et~al.(2023)Hahn, Perez, Schewe, Somenzi, Trivedi, and
  Wojtczak}]{Tacas23MJ}
Hahn, E.~M.; Perez, M.; Schewe, S.; Somenzi, F.; Trivedi, A.; and Wojtczak, D.
  2023.
\newblock Mungojerrie: Linear-Time Objectives in Model-Free Reinforcement
  Learning.
\newblock In \emph{Tools and Algorithms for the Construction and Analysis of
  Systems}, 527--545.

\bibitem[{Kakade(2003)}]{kakade2003sample}
Kakade, S.~M. 2003.
\newblock \emph{On the sample complexity of reinforcement learning}.
\newblock University of London, University College London (United Kingdom).

\bibitem[{Kearns and Singh(2002)}]{kearns2002near}
Kearns, M.; and Singh, S. 2002.
\newblock Near-optimal reinforcement learning in polynomial time.
\newblock \emph{Machine learning}, 49(2): 209--232.

\bibitem[{Littman et~al.(2017)Littman, Topcu, Fu, Isbell, Wen, and
  MacGlashan}]{littman2017environment}
Littman, M.~L.; Topcu, U.; Fu, J.; Isbell, C.; Wen, M.; and MacGlashan, J.
  2017.
\newblock Environment-independent task specifications via {GLTL}.
\newblock \emph{arXiv preprint arXiv:1704.04341}.

\bibitem[{Mnih et~al.(2015)Mnih, Kavukcuoglu, Silver, Rusu, Veness, Bellemare,
  Graves, Riedmiller, Fidjeland, Ostrovski et~al.}]{mnih2015human}
Mnih, V.; Kavukcuoglu, K.; Silver, D.; Rusu, A.~A.; Veness, J.; Bellemare,
  M.~G.; Graves, A.; Riedmiller, M.; Fidjeland, A.~K.; Ostrovski, G.; et~al.
  2015.
\newblock Human-level control through deep reinforcement learning.
\newblock \emph{nature}, 518(7540): 529--533.

\bibitem[{Silver et~al.(2016)}]{Silver16}
Silver, D.; et~al. 2016.
\newblock Mastering the game of {Go} with deep neural networks and tree search.
\newblock \emph{Nature}, 529: 484--489.

\bibitem[{Sutton and Barto(2018)}]{sutton2018reinforcement}
Sutton, R.~S.; and Barto, A.~G. 2018.
\newblock \emph{Reinforcement learning: An introduction}.
\newblock MIT press.

\bibitem[{Valiant(1984)}]{valiant1984theory}
Valiant, L.~G. 1984.
\newblock A theory of the learnable.
\newblock \emph{Communications of the ACM}, 27(11): 1134--1142.

\bibitem[{Voloshin et~al.(2022)Voloshin, Le, Chaudhuri, and
  Yue}]{voloshin2022policy}
Voloshin, C.; Le, H.; Chaudhuri, S.; and Yue, Y. 2022.
\newblock Policy optimization with linear temporal logic constraints.
\newblock \emph{Advances in Neural Information Processing Systems}, 35:
  17690--17702.

\bibitem[{Yang, Littman, and Carbin(2021)}]{yang2021tractability}
Yang, C.; Littman, M.; and Carbin, M. 2021.
\newblock On the (in) tractability of reinforcement learning for LTL
  objectives.
\newblock \emph{arXiv preprint arXiv:2111.12679}.

\end{thebibliography}

\newpage

\appendix

\section{Detailed Algorithm}

\begin{algorithm}
\caption{$\omega$-PAC}
\textbf{Input}: $|S|$, $|A|$, $T$, $\frac{1}{\varepsilon}$, $\frac{1}{\delta}$, and threshold $k > 0$ \\
\textbf{Output}: $6\varepsilon$-optimal policy $\pi$
\begin{algorithmic}[1]
\STATE Initialize all visit counts $c(s,a,s') \xleftarrow{} 0$
\STATE Let $c(s,a) = \sum_{s' \in S} c(s,a,s')$
\WHILE{true}
\STATE \textcolor{blue}{$\triangleright$ \textbf{Update $\Mhat$ with threshold $k$}}
\FOR{$(s,a)$ in $S \times A$}
\IF{$c(s,a) < k$}
\STATE $\widehat P(s,a,s') \xleftarrow{} \mathds{1}(s' = \mathtt{sink})$
\ELSE
\STATE $\widehat P(s,a,s') \xleftarrow{} \frac{c(s,a,s')}{c(s,a)}$
\STATE $\widehat F(s) \xleftarrow{} F(s)$
\ENDIF
\ENDFOR
\STATE $\widehat P(\mathtt{sink},\cdot,s') \xleftarrow{} \mathds{1}(s' = \mathtt{sink})$
\STATE $\widehat F(\mathtt{sink}) \xleftarrow{} \texttt{accepting}$
\STATE $\Mhat \xleftarrow{} (S\cup\{\mathtt{sink}\}, A, \widehat P, s_0, \widehat F)$
\STATE \textcolor{blue}{$\triangleright$ \textbf{Update $\pi$}}
\STATE $\pi \xleftarrow{} \text{optimal positional policy in $\Mhat$}$
\STATE \textcolor{blue}{$\triangleright$ \textbf{Check if done}}
\STATE $S_T \xleftarrow{}$ states reachable with positive probability \\ \qquad in $T$ steps from $s_0$ under $\pi$ in $\Mhat$
\IF{$c(s,a) = k$ for all $(s, a)$ in $S_T \times A$}
\STATE \textbf{return} $\pi$
\ENDIF
\STATE \textcolor{blue}{$\triangleright$ \textbf{Sample}}
\STATE Sample trajectory $\tau \sim \{(s_0, a_0), \ldots, (s_{T-1}, a_{T-1}) \}$ with $\pi$ in $\M$
\FOR{$i = 0, 1, \ldots, T-1$}
\IF{$c(s_i,a_i) < k$}
\STATE $c(s_i,a_i,s_{i+1}) \xleftarrow{} c(s_i,a_i,s_{i+1}) + 1$
\ENDIF
\ENDFOR
\ENDWHILE
\end{algorithmic}
\end{algorithm}

\end{document}